\documentclass{article}

\PassOptionsToPackage{numbers, compress}{natbib}

\usepackage[preprint]{neurips_2025}

\usepackage[utf8]{inputenc} 
\usepackage[T1]{fontenc}    
\usepackage{hyperref}       
\usepackage{url}            
\usepackage{booktabs}       
\usepackage{amsfonts}       
\usepackage{hyphenat}       
\usepackage{nicefrac}       
\usepackage{microtype}      
\usepackage{xcolor}         
\usepackage{comment}       
\usepackage[english]{babel}
\usepackage{amsmath}
\usepackage{amssymb}
\usepackage{amsthm}
\theoremstyle{definition}
\newtheorem{definition}{Definition}[section]
\usepackage{graphicx} 
\usepackage{tabularx}
\usepackage[most]{tcolorbox}
\usepackage{enumitem}
\newtheorem{theorem}{Theorem}

\title{On the Measure of a Model:\\ From Intelligence to Generality}

\author{%
  Ruchira Dhar\thanks{Primary Author.} \\
  Department of Computer Science\\
  University of Copenhagen\\
  \texttt{rudh@di.ku.dk} \\
   \And
  Ninell Oldenburg \\
  Centre for Philosophy of AI\\
  University of Copenhagen\\
  \texttt{niol@hum.ku.dk} \\
  \AND
  Anders Søgaard\\
  Department of Computer Science\\
  University of Copenhagen\\
  \texttt{soegaard@di.ku.dk} \\
}

\begin{document}

\maketitle

\begin{abstract}

Benchmarks such as ARC, Raven-inspired tests, and the Blackbird Task are widely used to evaluate the {\em intelligence} of large language models (LLMs). Yet, the concept of intelligence remains elusive—lacking a stable definition and failing to predict performance on practical tasks such as question answering, summarization, or coding. Optimizing for such benchmarks risks misaligning evaluation with real-world utility. Our perspective is that evaluation should be grounded in {\em generality} rather than abstract notions of intelligence. We identify three assumptions that often underpin intelligence-focused evaluation: {\em generality}, {\em stability}, and {\em realism}. Through conceptual and formal analysis, we show that only generality withstands conceptual and empirical scrutiny. Intelligence is not what enables generality; generality is best understood as a multitask learning problem that directly links evaluation to measurable performance breadth and reliability. This perspective reframes how progress in AI should be assessed and proposes generality as a more stable foundation for evaluating capability across diverse and evolving tasks.

\end{abstract}

\section{Introduction}

Many see large language models (LLMs) \cite{radford2019language, brown2020language, burton2024large, steyvers2025large} as steps toward \textit{artificial intelligence} (AI), or even as early prototypes of \textit{artificial general intelligence} (AGI) \cite{bubeck2023sparks, zhong2024evaluation,ilic2024evidence} i.e, systems that may eventually achieve parity with human intelligence across a wide range of abilities. If we want to assess whether LLMs are becoming more capable, we first need to ask: What does it mean for an AI system to be ``intelligent''? While ``intelligence'' may be an attractive buzzword, it is conceptually vague across disciplines \cite{howard1993intelligence, legg2007collection, wang2019defining,mitchell2024debates} and has more recently lead to researchers pointing out how such unfounded assumptions can serve a distractors in the field \cite{blili2024unsocial,mueller2024myth, blili2025stop, hendrycks2025definition}. In empirical evaluation practice, the quest for intelligence is outsourced to benchmarks of intelligence-indicating tasks \cite{horn1968organization, simon2024identifying} such as pattern abstraction \cite{chollet2019measure}, reasoning \cite{srivastava2023beyond, kazemi2025big}, or even general knowledge \cite{phan2025humanity,hendrycks2020measuring, rein2024gpqa}. However, , as we show in Section~\ref{sec2}, even intelligence benchmarks fail to predict what we actually care about, i.e., which models are better in terms of human preference and performance on real-world tasks. Such disconnect suggests that intelligence, as it is currently framed, may offer a less stable foundation for evaluating modern AI systems than is often implicitly assumed.

Our main contribution in this work is 
to show that this search effort for better models might be redirected toward something more concrete and grounded. We identify the three implicit assumptions behind using intelligence as a yardstick for evaluating LLMs: that we seek general-purpose systems ({\bf generality}), that there exists a fixed set of tasks worth mastering ({\bf stability}), and that intelligence is a real, general capacity required to solve those tasks ({\bf realism}). We show how differing strands of current AI research embody these principles and demonstrate the independence and necessity of the generality principle. This analysis leads to the core tenet of our perspective: Generality offers a sufficient foundation for evaluation, whereas stability and realism have little support. We also outline the case for generality as a promising foundation for evaluation:  it sidesteps some of the conceptual challenges associated with intelligence and is supported by insights from multitask learning theory. Our aim with this work is to bring clarity to evaluation by making a simple point: rather than asking whether a model is ``intelligent'', we suggest focusing on how generally and reliably it performs, reframing evaluation around a more grounded notion of generality.

\section{Why Intelligence is Problematic}
\label{sec2}

The idea that some models are simply ``more intelligent'' than others has become central to how progress in language modeling is communicated \cite{bubeck2023sparks, morris2023levels}. Benchmarks such as ARC \cite{chollet2019measure}, Raven tests \cite{abdelkarim2025evaluating} or the Blackbird Task \cite{merlo-2023-blackbird} are often used to make such claims, implicitly treating benchmark performance as a proxy for general intelligence or capability. Yet, what these benchmarks actually measure is rarely interrogated, and our analysis suggests that the picture they paint might be incomplete and potentially misleading.

\subsection{Conceptual Instability}
\label{sec2.1}

The discussion about general intelligence originates with the controversy between Spearman and Thomson \cite{thomson1916}, but has since resurfaced within neuroscience \cite{sims2013theory},  cognitive science \cite {sternberg1990metaphors,pfeifer2001understanding,sternberg2005cognition}, and  education \cite{demetriou20051, ritchie2018much}---yet there is little consensus on what constitutes intelligence. The search for capabilities that reliably indicate intelligence is hindered by the fact that it remains a fluid and contested concept \cite{legg2007collection, holm2024intelligence} across disciplines. Efforts to anchor intelligence in neuroscience have fallen short \cite{mackintosh1986biology,sternberg2003biological,pietschnig2015meta, gignac2017brain}. 
Similarly, intelligence has long resisted a stable, unified definition in cognitive science. Far from converging on a shared theoretical framework, researchers across subfields have produced divergent, and often incompatible, models of what intelligence is. While earlier works have emphasized problem-solving and reasoning as hallmarks of intelligence \cite{simon1976computer}, recent research has indicated the distributed and task-specific nature of neural activation--- where intelligence is not a system in itself but an emergent property of reorganization of task-specific neural systems in the brain \cite{anderson2010neural, pessoa2014understanding, bola2015dynamic, thiele2022multitask}. Together, these perspectives reinforce a central point: Intelligence may not be a unitary, well-specified capacity, but a family of loosely related, context-sensitive abilities, shaped by cultural, developmental, neural, and environmental factors \cite{nisbett2001culture, sternberg2004culture, kan2013nature}. 
In AI, we have largely inherited this ambiguity. The focus has been on defining intelligence in relation to the abilities of systems to perform certain tasks in environments in ways similar to humans \cite{legg2007collection,poole2010artificial, russell2016artificial, kaplan2019siri, Bartneck2021}. Just as IQ Tests are incapable of indicating performance in significant cognitive capacities \cite{schonemann1983iq, detterman1989correlations, stanovich2009intelligence, raven1989standard,gould1996mismeasure,henrich2010weirdest}, intelligence benchmarks have been argued to be 
problematic \cite{bender2021dangers} and 
an unstable foundation for evaluating AI systems \cite{mitchell2024debates, pfister2025understanding}.

\subsection{Illusion of Competence}

\begin{figure}[!t]
  \centering
  \includegraphics [width=0.8\textwidth]{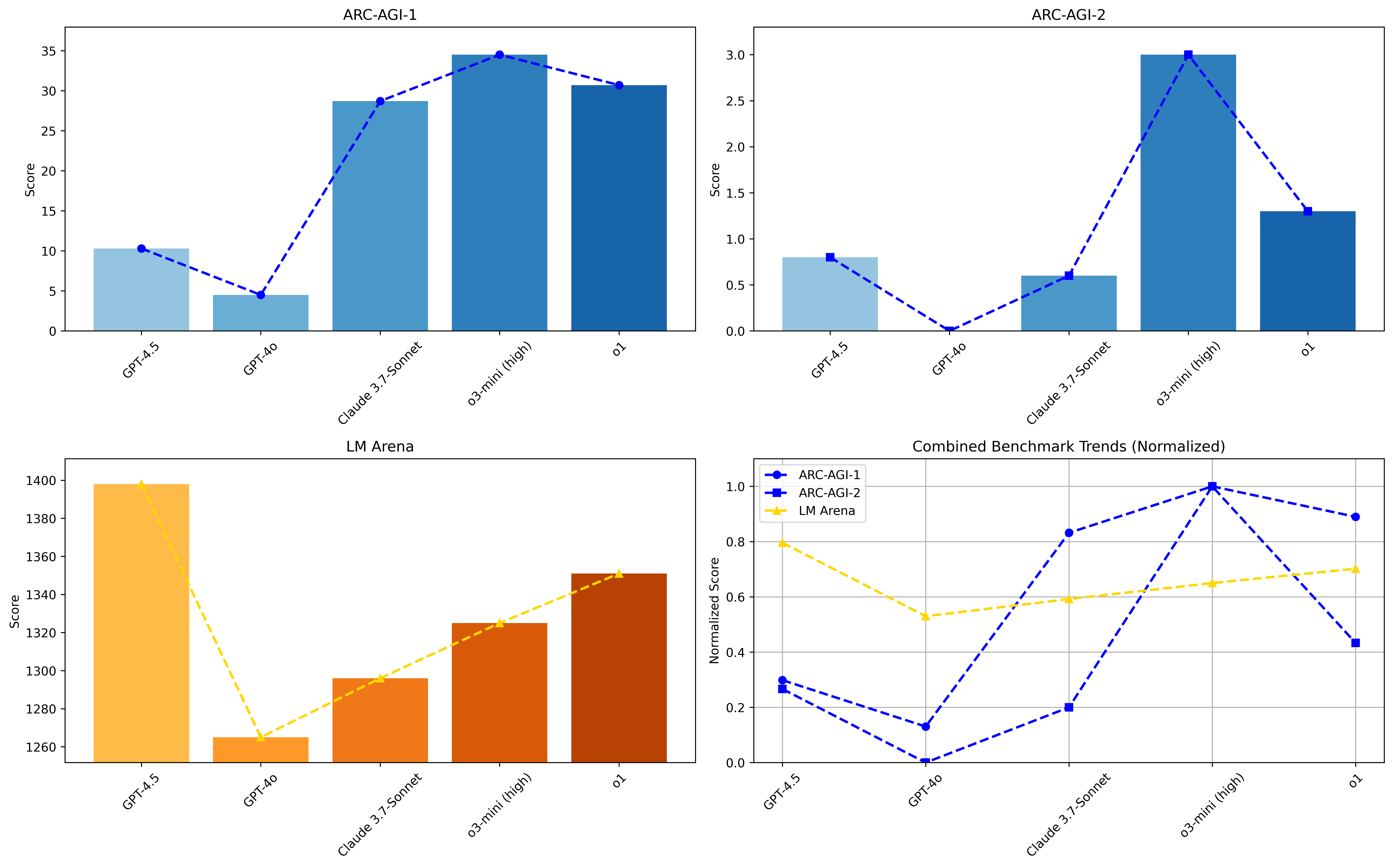}
  \caption{Comparing performance of LLMs on AGI benchmarks \texttt{ARC-AGI-1}, \texttt{ARC-AGI-2}  and preference-based benchmark like \texttt{LMArena}.  }
  \label{fig1}
\end{figure}

\begin{figure}[!t]
  \centering
  \includegraphics[width=0.8\textwidth]{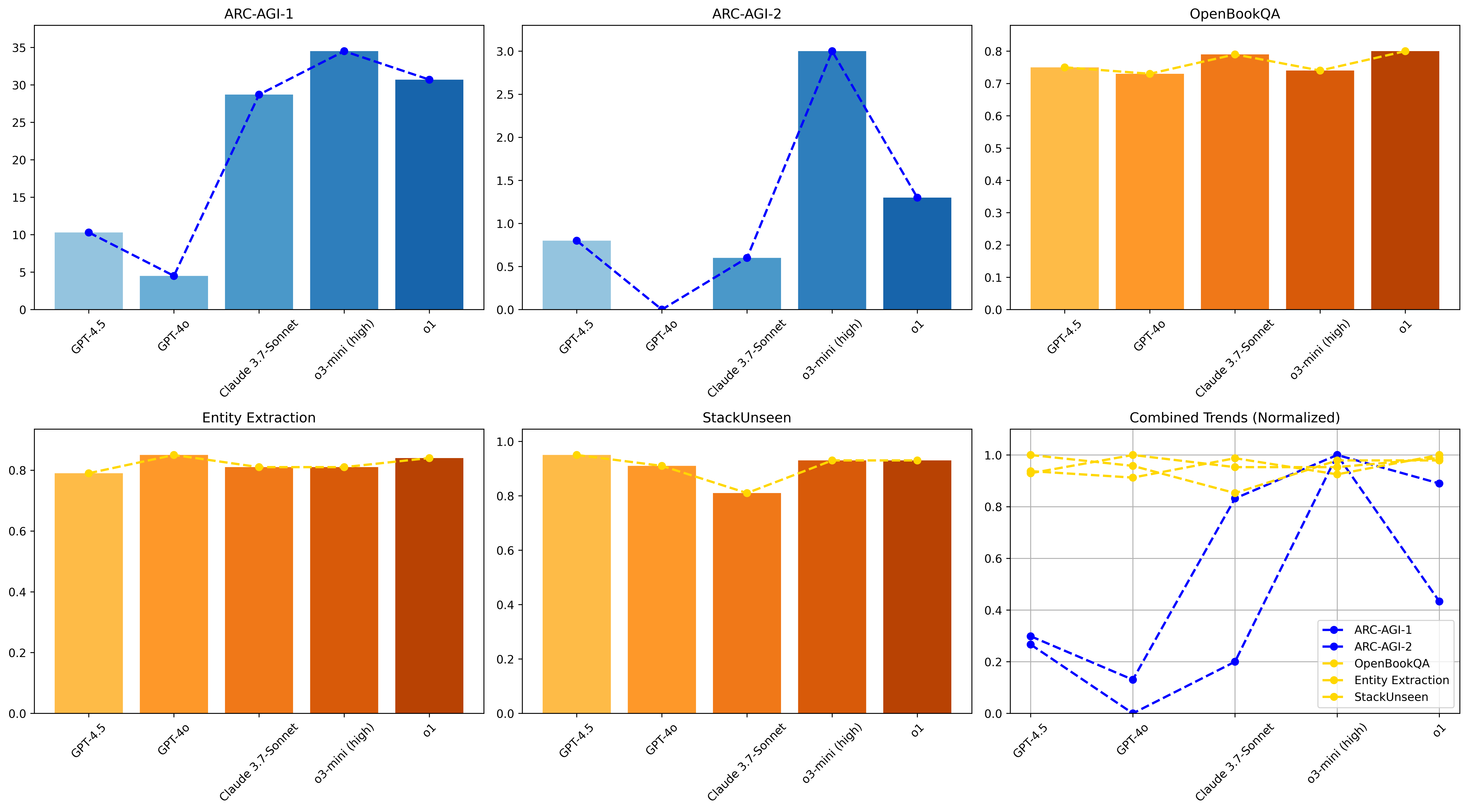}
  \caption{The performance of  LLMs on task-specific benchmarks \texttt{OpenBookQA, Entity Extraction, and StackUnseen. }}
  \label{fig2}
\end{figure}

Despite their widespread use, intelligence benchmarks often fail to reflect a model's effectiveness in real-world applications. We provide empirical evidence in \autoref{fig1} and \autoref{fig2} that strong performance on intelligence benchmarks exhibits limited correlation with human preference or task performance. As shown in \autoref{fig1}, the performance trends of different models for both ARC-AGI \cite{chollet2019measure} (considered frontier intelligence benchmarks) and LMArena \cite{chiang2024chatbot} differ significantly, i.e., a model that scores higher in ARC is not necessarily also better in LMArena (considered a human preference benchmark). Also, as seen in \autoref{fig2}, these performance trends do not translate into universal capability and into reliable performance across the kinds of real-world tasks language models are often used for \footnote{The data source and code is provided in Supplementary Materials.}. 
This undermines the assumption that intelligence benchmark scores are good predictors of general-purpose competence. 

\subsection{Rethinking our Ground}

The issues point to a deeper problem with how intelligence has been framed in AI evaluation. The continued reliance on intelligence-style benchmarks reflects an implicit belief that we are measuring something stable and meaningful, even as evidence suggests otherwise. This belief is sustained by a set of assumptions and biases that often go unstated \cite{gignac2015raven,chollet2019measureintelligence, hoffmann2022ai}. Benchmarks are treated as if they reflect an underlying cognitive trait, and success on some fixed tasks is seen as progress toward intelligence (but never quite reaching intelligence). These patterns persist because the ambiguity surrounding intelligence allows these assumptions to remain unchallenged. What, then, are we actually assuming when we say a model is becoming more intelligent? The next section will detail assumptions through evaluation studies.

\section{Unpacking Intelligence}
\label{sec3}

One of the earliest definitions of intelligence came from Marvin Minsky, a founding father of AI, and was as follows: ``Artificial Intelligence is the science of making machines do things that would require intelligence if done by people'' \cite{bolter1984turing,fjelland2020general}.  However, this leaves open the question of what intelligence is, and how we best make machines do tasks that require it. In the literature, we have two divergent commitments. The first comes from ideas on generality from John McCarthy (another founding father of AI),  which advocates development of system solutions that are independent of narrow problem domains \cite{10.1145/33447.33448}. This view is also exemplified from the rich and long literature on intelligence testing, where Legg and Hunter coalesce 70 definitions into a core tenet: ``Intelligence measures an agent's ability to achieve goals in a wide range of environments'' \cite{legg2007collection, chollet2019measureintelligence}. In contrast, we have the more current discourse  on intelligence as an emergent property that can be learned from data, a property that a model can possess \cite{bubeck2023sparks,chollet2019measure, Serpico2018-SERCTG,10.5555/3666122.3669277}. This dichotomy creates two distinct evaluation paradigms:

\begin{itemize}
    \item[i)] The first view emphasizes \textit{Generality}. A system is intelligent because it is able to achieve a wide range of goals. Intelligence is demonstrated directly by the scope of performance. \footnote{Here, generality refers to the empirical manifestation of task performance across environments, not an abstract property of ``generalization'' in the learning-theoretic sense \cite{kozlov2023ai, hupkes2023taxonomy}.}
    \item[ii)] The second view emphasizes \textit{Realism}. A system is intelligent because it possesses some latent property that allows it to achieve a wide range of goals. Intelligence is not only observed but posited as an explanatory trait.
\end{itemize}

Once we take the realist route, however, further commitments follow. Realism, i.e., the idea that intelligence refers to a fixed, real property, implies that the capacities unlocked by intelligence are fixed. This means that if we can enumerate these capacities, we can design representative task suites and treat performance there as evidence, i.e, a system is intelligent if it does well on a fixed representative set of tasks that embody intelligence \cite{chollet2019measureintelligence}. If not,  benchmarks only approximate intelligence, i.e, 
a system is intelligent by possessing intelligence \cite{ilievski2025aligning}, and the benchmark performance of a system is only a direct measure of its intelligence. \footnote{The first stance within realism is operationalist, while the latter is more empirical. A coherent operationalist stance would hold that purely empirical measures cannot demonstrate that different instruments measure the same property \cite{mari2023philosophical}.} Synthesizing these positions, we can see that the primary approaches to evaluating intelligence collapse onto three fundamental commitments:

\begin{tabular}{ll}
    {\sc Generality}&  We should develop a {\em all-purpose} or {\em multi-task} system .\\
    {\sc Stability} & A system should solve a {\em fixed} set of high-value intelligence indicating tasks.\\
    {\sc Realism}& We should focus on modelling \em{intelligence} itself.
    \end{tabular}

These assumptions often work in tandem, shaping how evaluation benchmarks are constructed and interpreted. Below, we clarify how current research embodies these assumptions.

\paragraph{Generality} The assumption of generality is that AI systems should ideally perform well across a broad range of tasks, rather than specialize narrowly. It has shaped the highly productive subfield of multi-task learning in the computer science research \cite{caruana1993multitask,caruana1997multitask, zhang2021survey} with concurrent applications for tasks in other scientific domains \cite{allenspach2024neural}. This assumption is also common in both the AGI literature \cite{morris2023levels} and in practical NLP benchmarking initiatives, like Big Bench \cite{srivastava2023beyond, kazemi2025big} or HELM \cite{bommasani2023holistic}, that prioritize cross-domain competence on multiple language tasks. Even the instruction-tuning of models like T0 \cite{sanh2022multitask}, FLAN-T5 \cite{longpre2023flan}, or OPT-IML \cite{iyer2022opt} all focus training on diverse prompts and task formulations with the explicit aim of cross-task generalization. These efforts reflect a growing recognition that narrow task performance is insufficient, especially as LLMs are increasingly expected to act as generalist agents \cite{hernandez2021general, moor2023foundation, zhang2024generalist}.
    
\paragraph{Stability} Stability assumes that there exists a fixed set of tasks on which evaluation can reliably represent intelligence or capability. Recent research has often tried to identify the fixed set of such tasks \cite{hendrycks2025definition} with explicit focus on few such tasks as reasoning \cite{ilic2024evidence, 10.5555/3600270.3602070, kojima2022large, morishita2024enhancing} or planning \cite{valmeekam2022large, valmeekam2023planning, valmeekam2023planbench} as important indicators of LLM performance. 
These benchmarks encourage a view of task stability: that the relevant cognitive capabilities can be captured by a durable suite of challenge benchmarks. 
    
\paragraph{Realism} Perhaps the most pervasive assumption is that benchmark success reflects a singular underlying cognitive trait called ``intelligence''. This can be seen in general tests of intelligence \cite{chollet2019measure,phan2025humanity, cai2025mm} or IQ-style comparisons \cite{pellert2024ai, huang2024measuring, abdelkarim2025evaluating}. Even when not stated explicitly, realism underwrites the very idea that progress on benchmarks reflects deeper model traits. This assumption allows authors to claim that outperforming humans on test sets implies cognitive parity or superiority. Yet, as discussed earlier (see Section~\ref{sec2.1}), the idea of intelligence as a unified capacity is conceptually unstable and treating it as a ground truth for evaluation may be both philosophically and practically misleading.

Next, we show why generality is independent and also necessary and sufficient for evaluating models. 

\subsection{Generality is Independent}

We want to argue that accepting {\sc Generality} should not automatically lead us also to accept {\sc Stability} and {\sc Realism}. \paragraph{Formal setup.}
Let $\mathcal{T}$ be a (possibly infinite) set of tasks endowed with a probability
measure $Q$ (the \emph{task environment}). 
Each model $M$ induces a measurable performance function
\begin{equation}
    f_M : \mathcal{T} \to [0,1], 
    \qquad 
    t \mapsto f_M(t),
\end{equation}
where $f_M(t)$ denotes the normalized performance of model $M$ on task $t$. We want to define what "evaluating M" means under the three alternative assumptions.

\begin{definition}[Generality]
The \emph{generality} of a model is its expected performance across the
task environment:
\begin{equation}
    E_G(M) \;=\; \mathbb{E}_{t \sim Q}\!\big[ f_M(t) \big].
\end{equation}
It assumes no fixed task set or latent variable, only performance averaged
over the environment~$Q$.
\end{definition}

\begin{definition}[Stability]
The \emph{stability} of a model is its aggregated performance on a fixed
benchmark subset~$S \subset \mathcal{T}$:
\begin{equation}
    E_S(M) \;=\; F\!\big( ( f_M(t) )_{t \in S} \big),
\end{equation}
where $F$ is a predetermined aggregation functional. It assumes that the
same benchmark tasks remain representative.
\end{definition}

\begin{definition}[Realism]
The \emph{realism} assumption posits a latent cognitive representation
$I(M) \in \mathbb{R}^k$ and task-specific decoding functions
$g_t : \mathbb{R}^k \!\to\! [0,1]$, such that performance derives from
this shared latent space:
\begin{equation}
    E_R(M) \;=\; \mathbb{E}_{t \sim Q}\!\big[ g_t(I(M)) \big].
\end{equation}
It assumes that observable task success reflects an underlying property,
interpreted as ``intelligence''.
\end{definition}

Consider the following thought experiment. 

\begin{tcolorbox}[
  colback=gray!5!white,
  colframe=black!75!black,
  title={\textbf{Thought Experiment: Three Robot Designers}},
  fonttitle=\bfseries,
  coltitle=black,
  boxrule=0.8pt,
  arc=4pt,
  left=6pt,
  right=6pt,
  top=6pt,
  bottom=6pt,
  width=\linewidth,
  sharp corners=south
]

Three engineers each design a robot to play \textbf{handball} ($t_H$), 
\textbf{badminton} ($t_B$), and \textbf{polo} ($t_P$).  
Each adopts one of the evaluation assumptions from the formal setup.

\vspace{0.7em}

\renewcommand{\arraystretch}{1.3}
\begin{tabularx}{\linewidth}{@{}lX@{}}
\textbf{Generality Engineer} & 
Optimizes $E_G(M)$ across the entire task environment 
$Q = \{t_H, t_B, t_P\}$, minimizing 
$\mathbb{E}_{t \sim Q}[\ell_t(x_M)]$. 
\\[-0.2em]
& \textit{Result:} Learns a representation that transfers flexibly across all tasks, 
without assuming a fixed subset or hidden essence. \\[0.4em]

\textbf{Stability Engineer} & 
Assumes the relevant tasks are known in advance, fixing 
$S = \{t_H, t_B\}$ and optimizing 
\[
E_S(M) = F\!\big( (f_M(t))_{t \in S} \big).
\]
\\[-0.2em]
& \textit{Result:} Excels on $t_H,t_B$ but fails on unseen tasks like $t_P$, showing brittleness when $Q$ shifts. \\[0.4em]

\textbf{Realism Engineer} & 
Postulates a latent ``athletic intelligence'' vector 
$I(M) \in \mathbb{R}^k$ and assumes that task performance derives from 
task-specific decoders $g_t$, such that 
\[
E_R(M) = \mathbb{E}_{t \sim Q}\!\big[ g_t(I(M)) \big].
\]
\\[-0.2em]
& \textit{Result:} Abstracts over tasks, but no finite $I(M)$ can capture the 
incompatible task gradients 
$\nabla\ell_{t_H}(x_H^\star)$, 
$\nabla\ell_{t_B}(x_B^\star)$, and 
$\nabla\ell_{t_P}(x_P^\star)$.
\end{tabularx}

\vspace{0.6em}
\textit{Conclusion:} Generality succeeds without presupposing either a fixed task set (stability) 
or a latent cognitive variable (realism), demonstrating its logical independence.

\end{tcolorbox}

This thought experiment illustrates a key insight: \textit{generality can be pursued without assuming either stability or realism}. One can seek broad capability without committing to a fixed task set or a unified latent construct. 

\subsection{(Only) Generality is Necessary}

The necessity of generality becomes clear when we ask what any evaluation aims to uncover.  Following Hernández-Orallo \textit{et al.} (2021) \cite{hernandez2021general}, evaluation can be expressed as a mapping from task to expected success, yielding an \emph{agent-characteristic curve}. In the formal setup above, each model $M$ induces a performance function 
$f_M : \mathcal{T} \to [0,1]$ that maps a task $t \in \mathcal{T}$ to its expected 
success.  Let $h(t)$ denote the \emph{difficulty} of task $t$, where 
$h : \mathcal{T} \to \mathbb{R}_{+}$ is monotonic with respect to the 
resources required to solve~$t$.  Aggregating performance over the 
distribution $Q$ of tasks at each difficulty level gives an 
\emph{agent–characteristic curve} (ACC)
\begin{equation}
    \psi_M(h) \;=\;
    \mathbb{E}_{t\sim Q \,|\, h(t)=h}\!\big[f_M(t)\big],
\end{equation}
which captures how success declines as difficulty increases.

\vspace{0.3em}
\noindent
\textbf{Why evaluation requires generality.}
If two systems $M_1$ and $M_2$ yield the same mean performance under~$Q$
but exhibit different $\psi(h)$ curves, they cannot be regarded as 
equivalent: one may fail exclusively on trivial tasks while the other fails 
only on the hardest ones.  Any evaluation rule that depends only on the 
mean of $f_M$ is therefore non-identifiable: it assigns the same score to 
agents with distinct behavioural profiles and cannot predict performance 
under a shifted or expanded difficulty distribution.  A meaningful evaluation 
must depend on the \emph{shape} of $\psi_M(h)$, not merely its average 
height.

Formally, define the \emph{spread} of the ACC as the dispersion of success
over difficulty:
\begin{equation}
    S_M^2
    \;=\;
    \int_0^\infty (h - \bar{h}_M)^2 \psi_M(h)\,dh,
    \qquad
    \bar{h}_M
    = \frac{\int_0^\infty h \psi_M(h)\,dh}
           {\int_0^\infty \psi_M(h)\,dh}.
\end{equation}
A small $S_M$ means that performance is concentrated on the easier 
portion of the domain before a clean decline—predictable and 
\emph{general} behaviour—whereas a large $S_M$ indicates scattered or 
specialised success.  The reciprocal quantity
\begin{equation}
    \Gamma_M = \frac{1}{S_M}
\end{equation}
is thus a direct measure of generality.\footnote{%
This formulation follows the generality–spread relation introduced by
Hernández-Orallo \textit{et al.} (2021).}

\vspace{0.3em}
\noindent
\textbf{Why only generality}
Without accounting for $\Gamma_M$, evaluation collapses into an 
unanchored average that changes arbitrarily with the sampling of tasks.  
By weighting performance by difficulty and including its concentration 
through~$S_M$, we obtain quantities that are invariant to task relabelling 
and stable under distributional shift.  In other words, once tasks are 
organised along a difficulty axis, the measure of how tightly performance 
is distributed across that axis becomes the only element that ensures 
comparability across environments.  Therefore, \emph{generality is the 
necessary condition for evaluation to be coherent and transferable}.  
Neither \textsc{Stability} nor \textsc{Realism} is required: fixing a task 
subset removes the difficulty structure, and positing a latent “intelligence” 
adds nothing to the observable shape of~$\psi_M(h)$.  Evaluation that is 
empirical, predictive and domain-invariant must, at minimum, 
recover~$\Gamma_M$.

\section{From Intelligence to Generality: What it Offers}
\label{sec4}

We have argued that generality is the only foundational assumption necessary for evaluating model capabilities because it does not rely on unresolved philosophical commitments and aligns with real-world deployment goals. We identify two factors in support of our position: generality is conceptually stable and it is theoretically grounded in multitask learning.

\paragraph{Conceptual Stability} Current evaluation practices often implicitly rely on unstable or unnecessary assumptions---particularly the ideas that intelligence is a unified cognitive property ({\sc Realism}), or that there exists a fixed set of core tasks capturing model ability ({\sc Stability}). In contrast, generality offers a conceptually stable and empirically grounded alternative---one that aligns more directly with how models are used and deployed. Prior work in cognitive science and AI emphasizes generalization as the hallmark of intelligent behavior \cite{lake2017building,tenenbaum2001generalization,yu2020meta,tomov2021multi, ilievski2025aligning}. Moreover, generality is flexible to task drift and evolving use cases. Unlike static benchmarks, which quickly lose relevance as models saturate their task sets or learn test-specific heuristics, generality-based evaluation can accommodate new tasks as they emerge. It requires only that we specify a diverse and representative sample of tasks at evaluation time---not that we define a canonical set of “core” challenges in advance. In short, generality is not only the most practically relevant evaluation principle, but it is also the most conceptually resilient. 


\paragraph{Theoretical Grounding} Generality also has strong support from learning theory and empirical evidence---especially from the literature on multitask learning (MTL) \cite{caruana1993multitask,caruana1997multitask}, a setting that explicitly embraces task diversity and seeks to learn shared inductive structure across tasks. In humans, this is often described as ``learning how to learn'' \cite{thrun1998learning, ilievski2025aligning}: leveraging knowledge from previously encountered problems to generalize more effectively to new ones. We extend the classical results \cite{baxter2000model, maurer2006bounds} to show that in multi-task learning, the generalization bound is reduced by a fixed factor. This means that evaluating performance across multiple tasks is not just more comprehensive---it gives stronger guarantees about future performance on unseen but related tasks.

\begin{theorem}
Consider an environment $\mathcal{E}$ consisting of a distribution $Q$ over tasks, where each task $P \sim Q$ is a distribution over data in a learning problem. Let $\mathcal{H}$ be a hypothesis class, $L_P(h)$ be the loss on task $P$, and $L_Q(h)$ be the model's environment average error. Then, for any $\delta>0$ (where $\delta$ is the confidence parameter), with probability at least $1-\delta$, the generalization bound is reduced by approximately a factor of $\sqrt{n}$ in the multi-task case.
\end{theorem}
\begin{proof} We proceed in three steps: 

\textbf{Step 1: Generalization Bound for Single-Task Environment (STE)} Let $\mathbb{E}$ be an environment consisting of a distribution $Q$ over tasks. For each task $P \sim Q$, let true loss of $h \in \mathcal{H}$ be given by $L_P(h) = \mathbb{E}_{(x,y)\sim P}[\ell(h(x),y)]$ where $\ell$ is a loss function. The empirical loss computed from $m$ i.i.d.\ samples drawn from $P$: 

\begin{equation}
    \hat{L}_P(h) = \frac{1}{m} \sum_{i=1}^{m} \ell\bigl(h(x_i), y_i\bigr)
\end{equation}

and the environment--average loss is:$L_Q(h) = \mathbb{E}_{P\sim Q}[L_P(h)].$ By standard PAC-learning results (see \cite{baxter2000model}), with probability at least $1-\delta$:
\begin{equation}
    \sup_{h\in \mathcal{H}} |L_P(h) - \hat{L}_P(h)| = O\Bigg(\sqrt{\frac{C + \ln(1/\delta)}{m}}\Bigg).
\end{equation}

\textbf{Step 2: Generalization Bound for Multi-Task Environment (MTE)} Here we evaluate $h$ on $n$ tasks $P_1,\dots,P_n \sim Q$, each with $m$ samples, yielding the average empirical error as an estimate of $L_Q(h)$: $\frac{1}{n}\sum_{i=1}^n \hat{L}_{P_i}(h)$. 
There are two sources of generalization error:
\begin{enumerate}
    \item \textit{Within-task generalization:} By the same PAC bound as in STE, for each $P_i$, 
    \begin{equation}
        \sup_{h\in \mathcal{H}} |L_{P_i}(h) - \hat{L}_{P_i}(h)| = O\Bigg(\sqrt{\frac{C + \ln(n/\delta)}{m}}\Bigg).
    \end{equation}
    \item \textit{Across-task generalization:} Since tasks are drawn i.i.d. from $Q$, by Hoeffding’s inequality:
    \begin{equation}
        \sup_{h\in\mathcal{H}} |L_Q(h) - M(h)| = O\Bigg(\sqrt{\frac{C + \ln(1/\delta)}{n}}\Bigg),
    \end{equation}
    where $M(h) = \frac{1}{n} \sum_{i=1}^{n} L_{P_i}(h)$ is the average true error across task. 
\end{enumerate}

Combining (7) and (8) via the triangle inequality (check \ref{appendix:a} for details),
\begin{equation}
    \sup_{h\in\mathcal{H}} \Big|L_Q(h) - \frac{1}{n} \sum_{i=1}^n \hat{L}_{P_i}(h)\Big| \leq O\Bigg(\sqrt{\frac{C + \ln(n/\delta)}{m}} + \sqrt{\frac{C + \ln(1/\delta)}{n}}\Bigg).
\end{equation}

\textbf{Step 3: Comparison of Bounds.} 
We derive the STE bound:
\begin{equation}
    L_Q(h) \leq \hat{L}_{P}(h) + O\Bigg(\sqrt{\frac{C + \ln(1/\delta)}{m}}\Bigg),
\end{equation}
We derive the MTE bound as:
\begin{equation}
    L_Q(h) \leq \frac{1}{n}\sum_{i=1}^n \hat{L}_{P_i}(h) + O\Bigg(\sqrt{\frac{C + \ln(1/\delta)}{n m}}\Bigg),
\end{equation}
\end{proof}

Thus, we see that in a learning environment where tasks are drawn i.i.d. from a distribution Q, the single–task generalization bound decays at a rate inversely proportional to $m$ (the number of samples in the task) while for multi-task environments, the error decays much faster at a rate of $1/\sqrt{mn}$ where $n$ is the number of tasks evaluated in our environment. This $\sqrt{n}$ reduction results from combining within-task PAC bounds with an across-task concentration (via Hoeffding’s inequality), thereby demonstrating that multi–task evaluation (or learning) effectively reduces the estimation variance. 

\begin{theorem} Averaging the empirical accuracy over $n$ independent tasks reduces the estimation error by approximately a factor of $\sqrt{n}$ compared to evaluating on a single task.
\end{theorem}

\begin{proof}
For each task $P$, standard concentration results (e.g., via Hoeffding’s inequality) ensure that the deviation between the true accuracy $A_P(h)=\mathbb{E}_{(x,y)\sim P}\bigl[\mathbf{1}\{h(x)=y\}\bigr]$
and its empirical counterpart 
\[
\hat{A}_P(h)=\frac{1}{m}\sum_{i=1}^{m}\mathbf{1}\{h(x_i)=y_i\}
\]
is of order 
\[
O\Bigg(\sqrt{\frac{C+\ln(1/\delta)}{m}}\Bigg),
\]
where $C$ reflects the complexity of $\mathcal{H}$. Since the tasks are independent, averaging over $n$ tasks gives the average empirical accuracy
\[
\hat{M}(h)=\frac{1}{n}\sum_{j=1}^{n}\hat{A}_{P_j}(h),
\]
which concentrates around the environment–average accuracy $A_Q(h)=\mathbb{E}_{P\sim Q}\bigl[A_P(h)\bigr]$ with a deviation of order
\[
\Big|A_Q(h)-\hat{M}(h)\Big|=O\Bigg(\sqrt{\frac{C+\ln(1/\delta)}{nm}}\Bigg).
\]
Thus, the estimation error is reduced by roughly a factor of $\sqrt{n}$ compared to the single-task case.
\end{proof}

\section{Alternative Perspectives}
\label{sec5}

We anticipate three recurring objections to our perspective. Each reflects a familiar intuition within the community, and we attempt to put these under closer scrutiny. 

\paragraph{We Are Already Doing Multitask Learning} While it is true that training paradigms today often reflect a mix of tasks and data sources \cite{sanh2022multitask,longpre2023flan, iyer2022opt}, this does not imply that multitask learning serves as a guiding principle in evaluation. Model capabilities are often interpreted through single-task type benchmarks, which are taken as evidence of intelligence \cite{chollet2019measure,phan2025humanity,rein2024gpqa}. Simultaneously, the discourse around AGI is increasingly popular \cite{bubeck2023sparks, morris2023levels, feng2024how, you2024far, seferis2024benchmark} and has moved goalposts from behavior to inferred property. If multitask learning were central, we would see evaluation grounded in formal principles such as task distributions, generalization error bounds, or transfer smoothness. But these are largely absent. 

\paragraph{We Don't Need Generality} If a domain-specific system works well, why prioritize generality? Yet the risks here are clear. Narrow systems, even if performant in specific contexts, are often brittle, redundant, and difficult to maintain across varied deployment contexts. The emergence of LLMs precisely reflects this \cite{radford2019language, brown2020language}. In dynamic environments where tasks drift, generality is not a luxury—it is a precondition for robustness \cite{sogaard-etal-2021-need}.

\paragraph{Intelligence is Real} A third objection begins from realism: the belief that “intelligence” is a unified cognitive capacity underlying performance. This view is intuitive—humans often excel across multiple tasks—but general behavior can instead emerge from overlapping task demands and shared inductive biases across learning problems \cite{thrun1998learning, baxter2000model, maurer2006bounds}. Per Occam’s razor, we don't need intelligence if the observed correlations among task performances can be explained without it. Intelligence may exist, but its invocation in evaluation adds conceptual weight without operational gains.

\section{Conclusion}
\label{sec6}

In this work, we have put forth the perspective that model evaluation should be grounded in \emph{generality}—the breadth and consistency of performance across tasks—rather than in abstract notions of \emph{intelligence}. Unlike intelligence, which rests on unstable conceptual and empirical foundations, generality offers a measurable, theoretically grounded, and operationally meaningful principle. It captures what truly matters for deployment: how reliably a system performs when tasks vary or evolve.

By showing that generality is both independent from, and sufficient for, coherent evaluation, we provide a framework that unifies conceptual clarity with formal rigor. This reframing is increasingly necessary as language models are applied in open and shifting environments, where success cannot be defined by static benchmarks \cite{kiela2021dynabench, hofmann2025fluid, kim2025benchmark} or latent cognitive claims \cite{gignac2015raven, blili2025stop} . Future progress in AI should therefore be assessed not by how “intelligent’’ a model appears, but by how \emph{generally and dependably} it performs across the diverse tasks we ask of it.

\bibliography{neurips_2025}

\newpage
\appendix
\label{appendix:a}
\section{Applying Triangle Inequality}

To derive the final generalization bound, we apply the 	\textbf{triangle inequality}. We aim to bound the difference between the 	\textbf{true expected loss} over all tasks, $L_Q(h)$, and the 	\textbf{empirical loss estimate} based on sampled tasks and finite data:

\begin{equation}
    \sup_{h\in\mathcal{H}} \Big|L_Q(h) - \frac{1}{n} \sum_{i=1}^{n} \hat{L}_{P_i}(h)\Big|.
\end{equation}

Using the definition of the empirical mean of true losses across sampled tasks,
\begin{equation}
    M(h) = \frac{1}{n} \sum_{i=1}^{n} L_{P_i}(h),
\end{equation}
We rewrite the expression as:

\begin{equation}
    \sup_{h\in\mathcal{H}} \Big|L_Q(h) - M(h) + M(h) - \frac{1}{n} \sum_{i=1}^{n} \hat{L}_{P_i}(h)\Big|.
\end{equation}

Applying the  inequality:

\begin{equation}
    \sup_{h\in\mathcal{H}} \Big|L_Q(h) - \frac{1}{n} \sum_{i=1}^{n} \hat{L}_{P_i}(h)\Big|
    \leq \sup_{h\in\mathcal{H}} |L_Q(h) - M(h)| + \sup_{h\in\mathcal{H}} |M(h) - \frac{1}{n} \sum_{i=1}^{n} \hat{L}_{P_i}(h)|.
\end{equation}

From the generalization bounds:

\begin{equation}
    \sup_{h\in\mathcal{H}} |L_Q(h) - M(h)| = O\Bigg(\sqrt{\frac{C + \ln(1/\delta)}{n}}\Bigg),
\end{equation}

\begin{equation}
    \sup_{h\in\mathcal{H}} |M(h) - \frac{1}{n} \sum_{i=1}^{n} \hat{L}_{P_i}(h)| = O\Bigg(\sqrt{\frac{C + \ln(n/\delta)}{m}}\Bigg),
\end{equation}

We combine these results to obtain:

\begin{equation}
    \sup_{h\in\mathcal{H}} \Big|L_Q(h) - \frac{1}{n} \sum_{i=1}^{n} \hat{L}_{P_i}(h)\Big| \leq O\Bigg(\sqrt{\frac{C + \ln(n/\delta)}{m}} + \sqrt{\frac{C + \ln(1/\delta)}{n}}\Bigg).
\end{equation}

This bound demonstrates that the overall generalization error consists of two independent components:

\begin{itemize}
    \item Within-task generalization error, which decreases as the number of samples per task ($m$) increases.
    \item Across-task generalization error, which decreases as the number of sampled tasks ($n$) increases.
\end{itemize}

Thus, we effectively combine these two error sources into a single, interpretable bound, ensuring a rigorous multi-task evaluation framework.

\newpage

\end{document}